\documentclass[letterpaper]{article}
\usepackage{uai2019}
\usepackage{natbib}
\usepackage[margin=1in]{geometry}

\usepackage{times}
\usepackage{hyperref}
\usepackage{url}
\usepackage[utf8]{inputenc}
\usepackage{amsmath} 
\usepackage{amsthm}
\usepackage{amsfonts}
\usepackage{bbm}
\usepackage{booktabs}
\usepackage{multirow}
\usepackage{array}
\usepackage{cuted}
\usepackage{hhline}
\usepackage{textcomp}
\usepackage{xcolor,colortbl}
\usepackage{subcaption}
\usepackage{graphicx}
\definecolor{greeen}{RGB}{180,216,189}
\definecolor{reed}{RGB}{255,217,212}
\usepackage{adjustbox}
\usepackage{hhline}
\usepackage[english]{babel}
\hypersetup{breaklinks=true}
\newcolumntype{x}[1]{>{\centering\let\newline\\\arraybackslash\hspace{0pt}}p{#1}}

\newtheorem{theorem}{Theorem}
\newtheorem{lemma}[theorem]{Lemma}

\title{Inhibited Softmax for Uncertainty Estimation in Neural Networks}

\author{ {\bf Marcin Możejko\thanks{\quad Authors contributed equally.}} \\
Sigmoidal \\
Warsaw, Poland
\And
{\bf Mateusz Susik\footnotemark[1]}  \\
Sigmoidal          \\
Warsaw, Poland \\
\And
{\bf Rafał Karczewski}   \\
Sigmoidal \\
Warsaw, Poland
}

\begin{document}

\maketitle

\begin{abstract}
We present a new method for uncertainty estimation and out-of-distribution detection in neural networks with softmax output. We extend the softmax layer with an additional constant input. The corresponding additional output is able to represent the uncertainty of the network. The proposed method requires neither additional parameters nor multiple forward passes nor input preprocessing nor out-of-distribution datasets. We show that our method performs comparably to more computationally expensive methods and outperforms baselines on our experiments from image recognition and sentiment analysis domains.
\end{abstract}

\section{Introduction}

The applications of computational learning systems might cause intrusive effects if we assume that predictions are always as accurate as during the experimental phase. Examples include misclassified traffic signs \citep{Evtimov18} and an image tagger that classified two African Americans as gorillas \citep{Curtis15}. This is often caused by overconfidence of models that has been observed in the case of deep neural networks \citep{Guo17}. Such malfunctions can be prevented if we estimate correctly the uncertainty of the machine learning system. Beside AI safety, uncertainty is useful in the active learning setting in which the data collection process is expensive or time consuming \citep{Houlsby11, Rottmann18}.

While uncertainty estimation in neural networks is an active field of research, the current methods are rarely adopted. It is desirable to develop a method that does not create additional computational overhead. Such a method could be used in environments that focus on quick training and/or inference. If such a method is simple, the ease of implementation should encourage practitioners to develop danger-aware systems in their work.

We suggest a method that measures the uncertainty of the neural networks with a softmax output layer. We replace this layer with Inhibited Softmax layer \citep{Saito16}, and we show that it can be used to express the uncertainty of the model. In our experiments, the method outperforms baselines and performs comparably with more computationally expensive methods on the out-of-distribution detection task. 

We contribute with:

\begin{itemize}
\item The mathematical explanation on the reason why the additional Inhibited Softmax output can be interpreted as an uncertainty measure.
\item The additions to the Inhibited Softmax that improve its uncertainty approximation properties.
\item The benchmarks comparing Inhibited Softmax, baseline and contemporary methods for measuring uncertainty in neural networks.
\end{itemize}

\section{Related Work}

The certainty of classification models can be represented by the maximum of probabilities \citep{HendrycksGimpel16}. It has been shown, however, that deep neural networks are prone to the overconfidence problem \citep{Guo17}, and thus so simple a method might not measure certainty well.

The modern Bayesian Neural Networks \citep{Blundell15, Hernandez15, Louizos17, MalininGales18, Wang16, Hafner18, Zhang18, Khan18} aim to confront this issue by inferring distribution over the models' weights. This approach has been inspired by Bayesian approaches suggested as early as the nineties \citep{Buntine91, Neal92}. A very popular regularisation mean - dropout - also can be a source of approximate Bayesian inference \citep{GalGhahramani16}. Such technique, called Monte Carlo dropout \citep{GalGhahramani15}, belongs to the Bayesian Neural Networks environment and has been since used in the real-life scenarios \citep[e.g.][]{Leibig17}. In the Bayesian Neural Networks, the uncertainty is modelled by computing the predictive entropy or mutual information over the probabilities coming from stochastic predictions \citep{Smith18}.

Other methods to measure the uncertainty of neural networks include a non-Bayesian ensemble \citep{Lakshminarayanan17}, a student network that approximates the Monte Carlo
posterior predictive distribution \citep{Balan15}, modelling Markov chain Monte Carlo samples with a GAN \citep{Wang18}, Monte Carlo Batch Normalization \citep{Teye18} and the nearest neighbour analysis of penultimate layer embedding \citep{Mandelbaum17}.

The concept of uncertainty is not always considered as a homogeneous whole. Some of the authors distinguish two types of uncertainties that influence predictions of machine learning models \citep{Kendall17}: epistemic uncertainty and aleatoric uncertainty. Epistemic uncertainty represents the lack of knowledge about the source probability distribution of the data. This uncertainty can be reduced by increasing the size of the training data. Aleatoric uncertainty arises from homoscedastic, heteroscedastic and label noises and cannot be reduced by the model. We will follow another source \citep{MalininGales18} that defines the third type: distributional uncertainty. It appears when the test distribution differs from the training distribution, i.e. when new observations have different nature then the ones the model was trained on.

A popular benchmark for assessing the ability of the models to capture the distributional uncertainty is distinguishing the original test set from out-of-distribution dataset \citep{HendrycksGimpel16}. There are works that focus only on this type of uncertainty \citep{Lee18}. ODIN \citep{Liang18} does not require changing already existing network and relies on gradient-based input preprocessing. Another work \citep{DeVries18} is close to the functionality of our method, as it only adds a single densely connected layer and uses a single forward pass for a sample.

Bayesian neural networks are more computationally demanding as they usually require multiple stochastic passes and/or additional parameters to capture the priors specification.

Our method meets all the following criteria:
\begin{itemize}
\item No additional learnable parameters required.
\item Only single forward pass needed.
\item No additional out-of-distribution or adversarial observations required.
\item No input preprocessing.
\end{itemize}

Another work that meets these criteria is based on a Dirichlet interpretation of the softmax output \citep{Sensoy18} and subjective logic. The softmax function is replaced with a ReLU layer in order to model belief masses within the posterior distribution.

The technique we use, Inhibited Softmax, has been successfully used for the prediction of the background class in the task of extraction of the objects out of aerial imagery \citep{Saito16}. The original work does not mention other possible applications of this softmax modification.

\section{Inhibited Softmax}

In this section, we will define the Inhibited Softmax function. We will provide a mathematical rationale on why it can provide uncertainty estimation when used as the output function of a machine learning model. Later we will present adjustments which we have made to the model architecture when applying Inhibited Softmax to a multilayer neural network. 

Let $x\in \mathbb{R}^n$ and $a\in \mathbb{R}$, then $IS_a$ is a function which maps $\mathbb{R}^{n}$ to $\mathbb{R}^{n}$. The i-th output in Inhibited Softmax is equal to:

\begin{equation}
    IS_{a}(x)_{i} = \frac{\exp x_i}{\sum_{i=1}^{n}\exp x_i + \exp a} \in (0, 1).
\end{equation}

Following equation holds:

\begin{equation}
    IS_{a}(x)_{i} = S(x)_i P_a^c(x),
\end{equation}

where:

\begin{equation}
    P_a^c(x) = \frac{\sum_{i=1}^{n}\exp x_i}{\sum_{i=1}^{n}\exp x_i + \exp a} \in (0, 1).
\end{equation}

and $S(x)$ is the standard softmax function applied to vector $x$. We will later refer to $P_a^c(x)$ as the "certainty factor".

Now let's assume that $IS_a$ is the output of a multiclass classification model trained with the \textit{cross-entropy} loss function $l_{IS}$. Assuming that the true class of a given example is equal to $t$ the loss is equal to:

\begin{equation}
\begin{split}
    l_{IS}(x, t) &= -\log IS_a (x)_t =  -\log S(x)_t - \log P_a^c (x) = \\ 
    & l_S(x, t) - \log P_a^c (x),
\end{split}
\end{equation}

where $l_S$ is the \textit{cross-entropy} loss function for a model with a standard softmax output. The optimisation process both minimises classification error (given by $l_S$) and maximises the certainty factor $P_a^c(x)$ for all the training examples. This is the intuition that explains why Inhibited Softmax mights serve as an uncertainty estimator - $P_a^c(x)$ is maximised for the cases from the training distribution.

If $P_a^c$ estimates the certainty of the model, in order to provide a valid uncertainty score we introduce:

\begin{equation}
    P_a^u (x) = 1 - P_a^c (x) = \frac{\exp a}{\sum_{i=1}^{n}\exp x_i + \exp a},
\end{equation}

It is minimised during the optimisation process. It might be interpreted as an artificial softmax output from the additional channel.

\subsection{Adjustments and Regularisation}

Although $P_a^c$ is maximized during the optimisation process we would like to ensure that its high values are obtained:
\begin{itemize}
    \item only for the cases from the training distribution,
    \item solely because of the training process, and neither because of the trivial solutions nor accidental network structure.
\end{itemize}
Because of that, we applied the following network adjustments:

\begin{itemize}
    \item \textbf{Removing bias terms from the inhibited softmax layer\footnote{See \hyperref[appendix]{Appendix 4} for the derivations.}} If we assume that the $i-th$ input to inhibited softmax is the output of a linear layer, namely:
    
    \begin{equation}
    x_i = \theta_i x_p + b_i,
    \end{equation}
    
    where $x_p$ is the vector of activations from the penultimate layer and $\theta_i$ is a vector of weights. As the derivative of the $\log$ \textit{certainty} factor w.r.t. bias:
    
    \begin{equation}
    \dfrac{\delta \log P_a^c(x)}{\delta b_i} = S(x)_i - IS_a(x)_i > 0.   
    \end{equation}

    is always positive, then increasing the value of $P_a^c (x)$ using any kind of gradient method can be achieved by increasing the values of biases. Now because of the fact that: 
    \begin{equation}
    \dfrac{\delta l_{S}(x, t)}{\delta b_i} = S(x)_i - \mathbb{I}_{i = t},
    \end{equation}
    (where $\mathbb{I}$ is indicator function), the derivative of \textit{cross-entropy} loss w.r.t. to biases vector $b$ alongside all-ones vector $(1, \dots, 1)$ is equal to:
    \begin{equation}
    \begin{split}
        \dfrac{\delta l_{S}(x, t)}{\delta \left(b = (1, \dots, 1)\right)} = \sum_{i = 1}^{n}\dfrac{\delta l_{S}(x, t)}{\delta b_i} =\\
    \sum_{i=1}^{n} \left(S(x)_i - \mathbb{I}_{i = t}\right) \equiv 0,
    \end{split}
    \end{equation}
    what implies that $l_S$ is constant along direction $(1, \dots, 1)$ as a function of biases $b$. Because of that - it is possible to increase the value of $P_a^c$ without the change of the classification loss by increasing all parameters $b_i$ by the same positive value $\delta > 0$. Due to that the network training process might result in a trival solution where $P_a^c$ was maximized by maximization of parameters $b_i$.
    
    \item \textbf{Changing the activation function to a kernel function in the penultimate layer of the network} The kernel activations are significantly grater than $0$ only close to their modes. Therefore, they make its outputs noticeably greater from $0$ only for a narrow, learnable region $W$ in its input space when applied to the penultimate layer. As we removed biases in a final layer - the input to $IS_a$ is also noticeably different from $0$ only in this narrow region as a strictly linear transformation of the activations vector of the penultimate layer. We believe that this prevents $P_a^c$ from achieving huge values due to an accidental inner network structure - as we expect that the out-of-distribution cases will often fall outside of the region $W$. Beyond $W$, $P_a^c$ has a constant value of:
    
    \begin{equation}
    \begin{split}
    P_a^c\left((0, \dots, 0)\right) &= \frac{\sum_{i=1}^{n}\exp 0}{\sum_{i=1}^{n}\exp 0 + \exp a} = \\
    &\frac{n}{n + \exp a}.
    \end{split}
    \end{equation}
    
    which might be interpreted as a "base rate certainty" for unseen examples controlled by hyperparameter $a\in\mathbb{R}$.
    
    \item \textbf{Evidence regularisation} In order to encourage examples from the training distribution to fall into the region $W$, where $P_a^c$ might be maximised, we introduced the following regularisation term to the loss function:
    \begin{equation}
        l'_{IS}(x, t) = l_{IS}(x, t) - \lambda\|x_p\|_{1},
    \end{equation}
    
    where $\|x_p\|_1$ is an $l_1$ norm of the activations of the penultimate layer and $\lambda > 0$ is a regularisation hyperparameter. During the optimisation process $\|x_p\|_1$ is maximised. It makes the activations of the penultimate layer $w$ noticeably greater than $0$ and consequently the examples from the training distribution fall into the region $W$. 
\end{itemize}

These adjustments significantly increased the certainty estimation properties of Inhibited Softmax. The dependency between performance and applying these changes to the model architecture is presented in
\hyperref[appendix1]{Appendix 1}.

\section{Experiments}

Firstly, we visualize the method on an XOR toy example taken from \citep{DeVries18}. We used the same experiment setting as the original paper, but with the Inhibited Softmax instead of the method proposed there. Our network consists of 4 layers with 100 hidden units each followed by the output layer. The dataset consists of 500 training samples. We can observe that our method reasonably estimates aleatoric uncertainty \hyperref[fig:toyexample]{(Figure 1)}. The uncertainty is present in the regions between the classes and where the classes overlap. In case of a small overlap (no noise in the dataset), the network is highly confident. The larger the noise the wider the region where our method is uncertain.

\subsection{Benchmarks}

\newpage

We have compared various ways of estimating uncertainty in neural networks (hereinafter referred to as "methods"). For the benchmarks, we implement these methods on top of the same \textit{base} neural network. We use the following experiments to check their quality:
\begin{itemize}
\item Out-of-distribution (OOD) examples detection - following \citep{HendrycksGimpel16} we use ROC AUC and average precision (AP) metrics to check the classifier's ability to distinguish between the original test set and a dataset coming from another probability distribution. These experiments show whether the method measures well the distributional uncertainty on a small sample of out-of-distribution datasets.
\item Predictive performance experiment - given a dataset, we split it into train, test and validation sets. We report accuracy and negative log loss on the test set. Any method should not deteriorate the predictive performance of the network.
\item Wrong prediction detection - we expect that the more confident the model is, the more accurate its predictions on in-distribution dataset should be. In this experiment the ground truth labels are used to construct two classes after the prediction on the test dataset is performed. The classes represent the correctness of the classifier prediction. Then, the uncertainty measure is used to compute TPRs and FPRs. We report ROC AUC scores on this setting. This experiment shows whether the method measures well the combination of epistemic and aleatoric uncertainty on a small sample of datasets. In this experiment, we do not report average precision score, as it would be distorted by different levels of misclassification in the predictions.
\end{itemize}

\begin{strip}
\centering
\includegraphics[page=1,width=\textwidth]{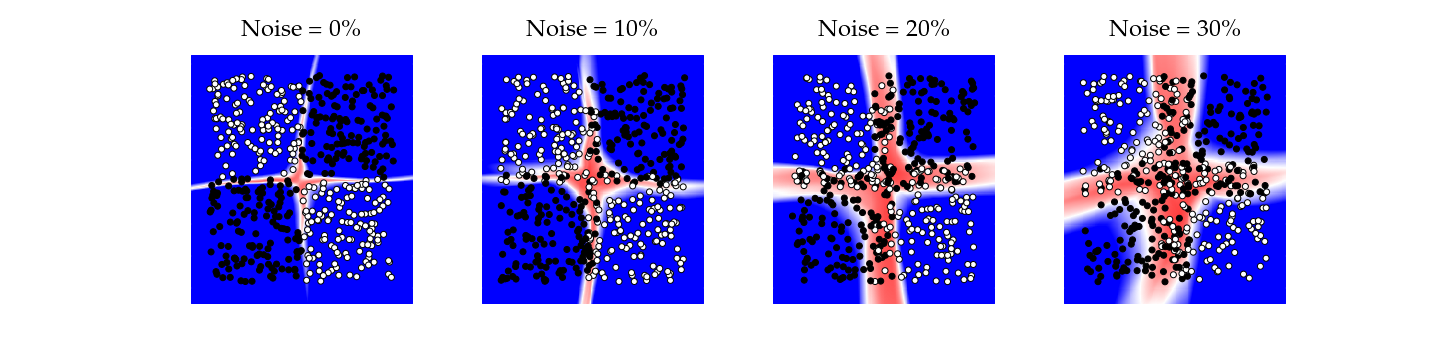}
\captionof{figure}{\label{fig:toyexample}Confidence predictions on XOR dataset. The evidence regularization constant was $10e^{-6}$, $a$ was 1. Blue indicates the largest confidence, red the largest uncertainty.}
\end{strip}

\clearpage

\begin{strip}
\centering
\begin{tabular}{ p{5cm} | p{5.7cm} | l }
\textbf{Method} & \textbf{Uncertainty measure} & \textbf{Abbreviation} \\\hline
Inhibited Softmax & value of the artificial softmax output & IS \\\hline
Base network & $1 - max(p_{i})$ & BASE \\\hline
Base network & entropy of the probabilities & BASEE \\\hline
Monte Carlo Dropout \citep{GalGhahramani16} & predictive entropy of the probabilities from 50 stochastic forward passes & MCD \\\hline
Bayes By Backprop with a Gaussian prior \citep{Blundell15} & predictive entropy of the probabilities from 10 stochastic forward passes & BBP \\\hline
Deep Ensembles without adversarial training \citep{Lakshminarayanan17} & predictive entropy of the probabilities from 5 base neural networks & DE
\end{tabular}
\captionof{table}{\label{tab:methods} The methods used for benchmarks. Both the base network methods will serve as baselines.}
\end{strip}

\begin{strip}
\centering
\begin{tabular}{ p{2.4cm} | p{4.8cm} | p{5.1cm} }
\textbf{In-distribution dataset} & \textbf{Out-of-distribution datasets} & \textbf{Base network} \\\hline

CIFAR-10 \citep{Krizhevsky09} & SVHN \citep{Netzer11} \newline LFW-A \citep{LearnedMiller15} & Custom small network trained with Adadelta \citep{Zeiler12} \\\hline
MNIST & NOTMNIST \citep{Bulatov11} \newline black and white CIFAR-10 \newline Omniglot \citep{Changala15} & Lenet-5 \citep{LeCun98} with an average pooling instead of a subsampling and a softmax layer instead of a gaussian connection trained with Adadelta \citep{Zeiler12} \\\hline

IMDB \citep{Maas11} & Customer Reviews \citep{Hu04} \newline Movie Reviews \citep{Pang04} \newline Reuters-21578 & Linear classifier on top of an embedding \citep[as in][]{HendrycksGimpel16} trained with RMSProp \citep{Tieleman12}
\end{tabular}
\captionof{table}{\label{tab:exp} Datasets and neural architectures used.}
\end{strip}

\hyperref[tab:methods]{Table 1} shows the methods and respective uncertainty measures that will be benchmarked\footnote{The choice of hyperparameters and training details for methods other than Inhibited Softmax is further discussed in the appendix. The implementation is available at: \url{https://github.com/MSusik/Inhibited-softmax}}. We establish two baselines. Both of them work on the unmodified base neural network, but uncertainty is measured in different ways, using either the maximum of probabilities over classes or entropy of probabilities. The method we suggest to use is referred to as \textit{IS}.

We have chosen these methods as they have been already used for benchmarking \citep[e.g.][]{Louizos17}, and they are well-known in the Bayesian Neural Network community. In the case of Inhibited Softmax we set evidence regularisation to $10^{-6}$, $a$ to 1 and we use rescaled Cauchy distribution's PDF ($f(x) = \frac{1}{1 + x^2}$). The datasets\footnote{Preprocessing is discussed in the appendix.} and the respective base neural networks we have chosen for the experiments are reported in \hyperref[tab:exp]{Table 2}.

The base network for CIFAR-10 consists of 3 2D convolutional layers with a 2D batch norm and 0.25 dropout. The convolving filter size was 3. Each convolutional layer was followed by 2D maximum pooling over 3x3 neurons with stride 2. The number of filters in the consecutive layers are 80, 160 and 240. Then there are 3 fully-connected layers. After the first fully-connected layer we apply 0.25 dropout. The number of neurons in the consecutive dense layers are 200, 100, 10.


In the experiments, we report averages over three training and prediction procedures on the same training-test splits. 

\newpage

\begin{strip}
\centering
\begin{adjustbox}{center}
\begin{tabular}{l | l | c | c | c | c | c | c}

\textbf{Datasets (In/Out)} & \textbf{Score} & \textbf{MCD} & \textbf{IS}  & \textbf{BASE} & \textbf{BASEE} & \textbf{BBP} & \textbf{DE} \\\hhline{*8-}
\quad MNIST/ & ROC  & 0.974 & 0.973 & 0.958 & 0.956 & \cellcolor{greeen}0.982 & 0.979 \\\hhline{~|*7-}
NOTMNIST & AP & 0.984 & 0.983 & 0.938 & 0.955 & \cellcolor{greeen}0.989 & 0.988 \\\hhline{*8-}
\quad MNIST/ & ROC & \cellcolor{greeen}0.999 & 0.998 & 0.997 & 0.996 & \cellcolor{greeen}0.999 & 0.999 \\\hhline{~|*7-}
CIFAR-10 B\&W & AP & \cellcolor{greeen}0.9997 & 0.9995 & 0.9994 & 0.999 & \cellcolor{greeen}0.9997 & \cellcolor{greeen}0.9997 \\\hhline{*8-}
\quad MNIST/ & ROC & 0.977 & \cellcolor{greeen}0.978 & 0.956 & 0.953 & 0.975 & 0.977 \\\hhline{~|*7-}
Omniglot & AP & \cellcolor{greeen}0.992 & \cellcolor{greeen}0.992 & 0.983 & 0.981 & 0.991 & 0.99 \\\hhline{*8-}
\quad CIFAR-10/ & ROC & 0.927 & 0.931 & 0.866 & 0.865 & 0.913 & \cellcolor{greeen}0.946 \\\hhline{~|*7-}
SVHN & AP & 0.987 & 0.986 & 0.961 & 0.958 & 0.981 & \cellcolor{greeen}0.99 \\\hhline{*8-}
\quad CIFAR-10/ & ROC & 0.693 & 0.74 & 0.593 & 0.594 & 0.723 & \cellcolor{greeen}0.755 \\\hhline{~|*7-}
LFW-A & AP & 0.142 & 0.174 & 0.127 & 0.126 & 0.169 & \cellcolor{greeen}0.181 \\\hhline{*8-}
\quad IMDB/ & ROC & 0.723 & \cellcolor{greeen}0.736
& 0.717 & 0.717 & 0.729 & 0.718 \\\hhline{~|*7-}
Customer Reviews & AP & 0.027 & \cellcolor{greeen}0.088 & 0.027 & 0.027 & 0.028 & 0.026 \\\hhline{*8-}
\quad IMDB/ & ROC & \cellcolor{reed}0.836 & \cellcolor{greeen}0.877 & 0.837 & 0.837 & 0.845 & \cellcolor{reed}0.835 \\\hhline{~|*7-}
Movie Reviews & AP & \cellcolor{reed}0.755 & \cellcolor{greeen}0.876 & 0.756 & 0.756 & 0.769 & \cellcolor{reed}0.753 \\\hhline{*8-}
\quad IMDB/ & ROC & 0.817 & \cellcolor{greeen}0.829 & 0.816 & 0.816 & \cellcolor{reed}0.805 & \cellcolor{reed}0.815 \\\hhline{~|*7-}
Reuters-21578 & AP & 0.735 & \cellcolor{greeen}0.82 & 0.727 & 0.727 & \cellcolor{reed}0.715 & \cellcolor{reed}0.724

\end{tabular}
\end{adjustbox}
\captionof{table}{\label{tab:out} Out of distribution detection results. The green colour shows the best results, the red - results worse than any of the baselines.}

\end{strip}


In all of the selected computer vision OOD tasks, Inhibited Softmax improves upon baselines. IS is better than BASE on \textit{NOTMNIST} (0.973 ROC AUC vs 0.958) and Omniglot (0.978 ROC AUC vs 0.956). IS' ROC AUC performance on \textit{MNIST/NOTMNIST} and \textit{CIFAR-10/SVHN} is similar to MCD (resp. 0.973 vs 0.974 and 0.931 vs 0.927). IS achieves a very good result on the \textit{CIFAR-10/LFW-A} task. In the task of discriminating \textit{MNIST} from black and white \textit{CIFAR-10} \hyperref[tab:out]{(Table 3)} our method achieves very high detection performance (0.998 ROC AUC and 0.9995 AP). This dataset is the least similar to \textit{MNIST}. In contrast to other datasets tested against the digit recognition networks, various shades of gray dominate the images. All the Bayesian methods vastly outperform the baselines on the computer vision tasks.

Inhibited Softmax improves upon other methods on the sentiment analysis task. Especially large improvement can be observed on the test against the \textit{Movie Reviews} dataset. For example, the ROC AUC of IS (0.877) is much greater than the ROC AUC of MCD (0.836). Methods other than IS are not much better than the baseline (BBP's 0.845 ROC AUC), sometimes being insignificantly worse (DE's 0.835 ROC AUC). IS is also the best on the test against \textit{Reuters-21578} and \textit{Customer reviews} (resp. 0.829 and 0.736). Two baselines achieve the same results on sentiment analysis experiment as there is no difference in ranking of the examples between the chosen uncertainty measures. We do not corroborate the results from the baseline publication \citep{HendrycksGimpel16}. We discovered that in that paper the out-of-distribution samples for \textit{Movie Reviews} were constructed by taking single lines from the dataset file, while the reviews span over a few lines. Our results show that the detection is a tougher task when full reviews are used (BASE achieves 0.837 ROC AUC vs 0.94 ROC AUC \citep{HendrycksGimpel16}).

In our experiments, the Inhibited Softmax does not deteriorate significantly the predictive performance of the neural network \hyperref[tab:predictive]{(Table 4)}. Its accuracy was similar to the baselines on every task, for example on \textit{IMDB} is 0.2\% lower and on \textit{CIFAR-10} dataset the accuracy is 0.5\% lower. Ensembling the networks gives the best predictive performance. We observed that all the text models perform classification very well on the \textit{Movie Reviews} dataset. Despite coming from a different probability distribution this dataset contains strong sentiment retrieved by the networks for the prediction of the correct label. It shows the generalization ability of the networks.

Wrong prediction detection results \hyperref[tab:wrong-prediction]{(Table 5)} show that IS is the only method that is able to detect misclassified observations better than a random classifier (0.689 ROC AUC) on the sentiment task. All the methods improve slightly over the baselines on the \textit{MNIST} dataset with DE improving the most (0.987) and \textit{BBP} improving the least (0.979). The Inhibited Softmax and Monte Carlo Dropout are worse than the baseline on \textit{CIFAR-10} (resp. 0.869 and 0.85 vs 0.875).  

\newpage

\begin{strip}
\centering
\begin{tabular}{l | l | c | x{2cm} | c | c | c}

\multicolumn{2}{l|}{} & \textbf{MCD} & \textbf{IS} & \textbf{BASE(E)} & \textbf{BBP} & \textbf{DE} \\\hhline{*7-}
\multirow{ 2}{*}{MNIST} & Accuracy & 0.992 & \cellcolor{reed}0.991 & 0.992 & \cellcolor{reed}0.991 & \cellcolor{greeen}0.994 \\\hhline{~|*6-}
 & NLL & \cellcolor{reed}0.034 & 0.03 & 0.035 & 0.031 & \cellcolor{greeen}0.019 \\\hhline{*7-}
\multirow{ 2}{*}{CIFAR10} & Accuracy & 0.854 & \cellcolor{reed}0.846 & 0.851 & \cellcolor{reed}0.841 & \cellcolor{greeen}0.88 \\\hhline{~|*6-}
 & NLL & 0.527 & 0.62 & 1.661 & 0.514 & \cellcolor{greeen}0.385 \\\hhline{*7-}
\multirow{ 2}{*}{IMDB} & Accuracy & 0.883 & \cellcolor{reed}0.881 & 0.883 & \cellcolor{reed}0.882 & \cellcolor{greeen}0.885 \\\hhline{~|*6-}
 & NLL & 0.291 & \cellcolor{reed}0.304 & 0.295 & \cellcolor{reed}0.302 & \cellcolor{greeen}0.289 \\\hhline{*7-}
IMDB model & Accuracy & \cellcolor{reed}0.848 & \cellcolor{reed}0.852 & \cellcolor{greeen}0.857 & \cellcolor{reed}0.849 & \cellcolor{reed}0.851 \\\hhline{~|*6-}
on Movie Reviews & NLL & \cellcolor{reed}0.378 & \cellcolor{greeen}0.346 & 0.362 & \cellcolor{reed}0.365 & \cellcolor{reed}0.586 \\\hhline{*7-} 
\multicolumn{2}{l|}{Number of forward passes} & 50 & 1 & 1 & 10 & 1 \\\hhline{*7-}
\multicolumn{2}{l|}{Params (vs BASE)} & x & \texttildelow x (no bias in the last layer) & x & \texttildelow 2x & 5x

\end{tabular}%
\captionof{table}{\label{tab:predictive} Predictive performance experiment results and computational overhead. Only the baseline and Inhibited Softmax have neither additional parameters nor require multiple forward passes. The green color shows the best results, the red - results worse than the baseline. We compare Inhibited Softmax and baselines with methods that require more forward passes and/or more params.}
\end{strip}

\begin{table}
\resizebox{\columnwidth}{!}{%
\begin{tabular}{l | l | c | c | c | c | c }
\textbf{Dataset} & \textbf{MCD} & \textbf{IS} & \textbf{BASE} & \textbf{BASEE} & \textbf{BBP} & \textbf{DE} \\\hhline{*7-}
MNIST & 0.982 & 0.983 & 0.979 & 0.979 & 0.979 & \cellcolor{greeen}0.987 \\\hhline{*7-}
CIFAR-10 & \cellcolor{reed}0.869 & \cellcolor{reed}0.85 & 0.875 & 0.877 & 0.878 & \cellcolor{greeen}0.886 \\\hhline{*7-}
IMDB & \cellcolor{reed}0.418 & \cellcolor{greeen}0.689 & 0.501 & 0.501 & \cellcolor{reed}0.398 & \cellcolor{reed}0.391
\end{tabular}%
}
\caption{\label{tab:wrong-prediction} Wrong prediction detection results (ROC AUC). The green color shows the best results, the red - results worse than any of the baselines.}
\end{table}

\section{Visualisation}

In practice, the overlap of the correctly detected out-of-distribution observations between Inhibited Softmax and Bayesian methods is surprisingly large. To demonstrate it, we compare Monte Carlo dropout and our method on an experiment from \citep{Smith18}. We train a fully connected variational autoencoder (VAE) on the \textit{MNIST} dataset. Then, we create a grid in the latent space and for each point we generate a sample. We plot the uncertainty estimation of the methods on generated samples from these points together with the labelled latent encoding of the test samples \hyperref[fig:vis]{(Figure 3)}. Both methods are unable to detect out of distribution samples generated from the bottom left corner of the 2D latent space. Another example of the similarity is that both of the methods do not estimate high uncertainty in the area where blue and purple classes intersect in the latent space.

This leads to a hypothesis that there exist samples that are tougher to detect by uncertainty measures for any recently proposed method. Similarly to the ideas from adversarial attacks field, it might be worth to investigate how to construct such samples. We believe it might be a way to improve uncertainty sampling performance.

\quad \\

\section{Further Work \& Limitations}

We notice that working on following aspects can enhance the uncertainty estimation:
\begin{itemize}
\item Developing an analogous to IS method for regression.
\item Limiting the number of required hyperparameters for Inhibited Softmax.
\item Expanding the method to hidden layers. This is especially promising as the Inhibited Softmax performs better than other methods on a shallow network in our sentiment analysis experiment. On deeper networks IS has not yet such advantage and it might be possible to outperform other methods.
\item Applying Inhibited Softmax to larger architectures and larger variety of machine learning tasks (e.g. sophisticated NLP models, reinforcement learning).
\end{itemize}

Although we showed experimentally that the architecture adjustments applied to the network architecture are beneficial, we are still lacking the full and sound mathematical explanation of their influence on model behaviour. Especially important is the explanation of the Inhibited Softmax's capability of modelling the alleatoric uncertainty presented in \hyperref[fig:toyexample]{Figure 1} and \hyperref[fig:wrong-prediction]{Table 5}.

\begin{strip}
\centering
\captionsetup{type=figure}
\begin{subfigure}{.5\textwidth}
\includegraphics[page=1,width=\textwidth]{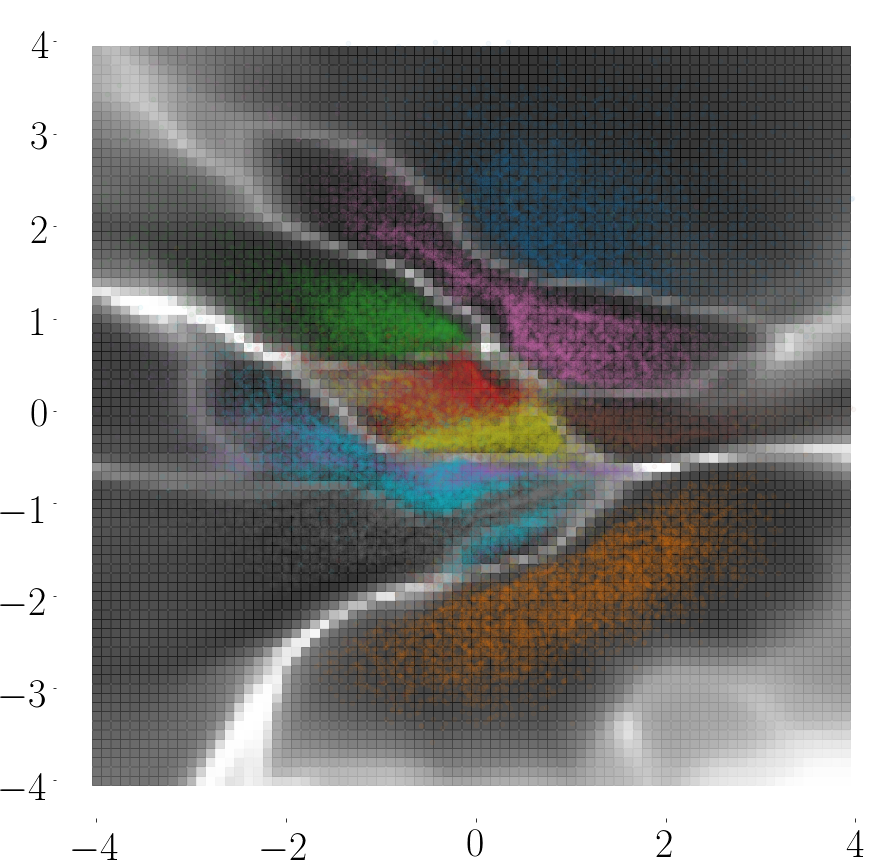}\\
\end{subfigure}%
\begin{subfigure}{0.5\textwidth}
\includegraphics[page=2,width=\textwidth]{vae_is}\\
\end{subfigure}
\begin{subfigure}{.5\textwidth}
\includegraphics[page=1,width=\textwidth]{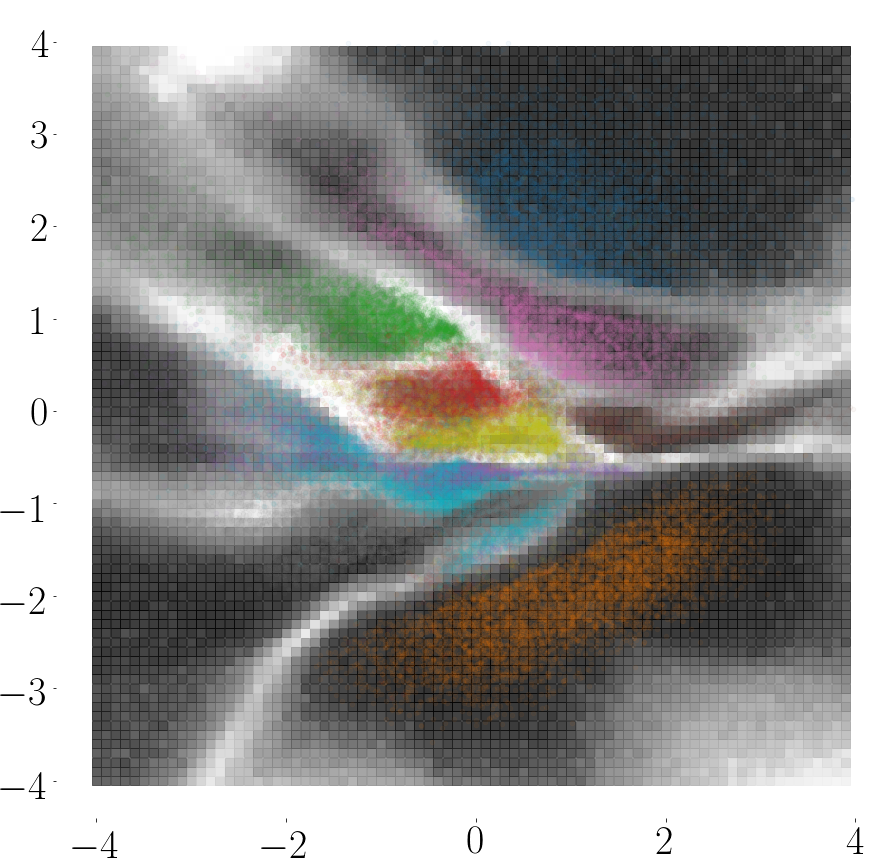}
\end{subfigure}%
\begin{subfigure}{0.5\textwidth}
\includegraphics[page=2,width=\textwidth]{vae_mc}
\end{subfigure}
\captionof{figure}{\label{fig:vis} Visualisation of uncertainty measures: Inhibited Softmax (top) and Monte Carlo Dropout (bottom) on the VAE's latent space. The shade of grey represent the normalised uncertainty on the samples generated from the latent space. The lighter the more uncertainty. The points represent the encoded test set and the colours are the classes. The axes show coordinates in the latent space. Note the similarity in the regions between the methods.}
\end{strip}

\section{Conclusion \& Discussion}


We reinterpreted Inhibited Softmax as a new method for uncertainty estimation. The method can be easily applied to various multilayer neural network architectures and does not require additional parameters, multiple stochastic forward passes or out-of-distribution examples.

The results show that the method outperforms baseline and performs comparably to the other methods. The method does not deteriorate the predictive performance of the classifier.

The predictive performance from \textit{IMDB/Movie Reviews} experiment suggests that even if the observation comes from another probability distribution and the uncertainty measure is able to detect it, the network can still serve as a useful classifier. Other way around, if we have an accurate classifier, it might happen that it shows uncertainty in its predictions.

\bibliography{iclr2019_conference}
\bibliographystyle{plainnat}

\clearpage

\begin{strip}

\includegraphics[page=1,width=\textwidth]{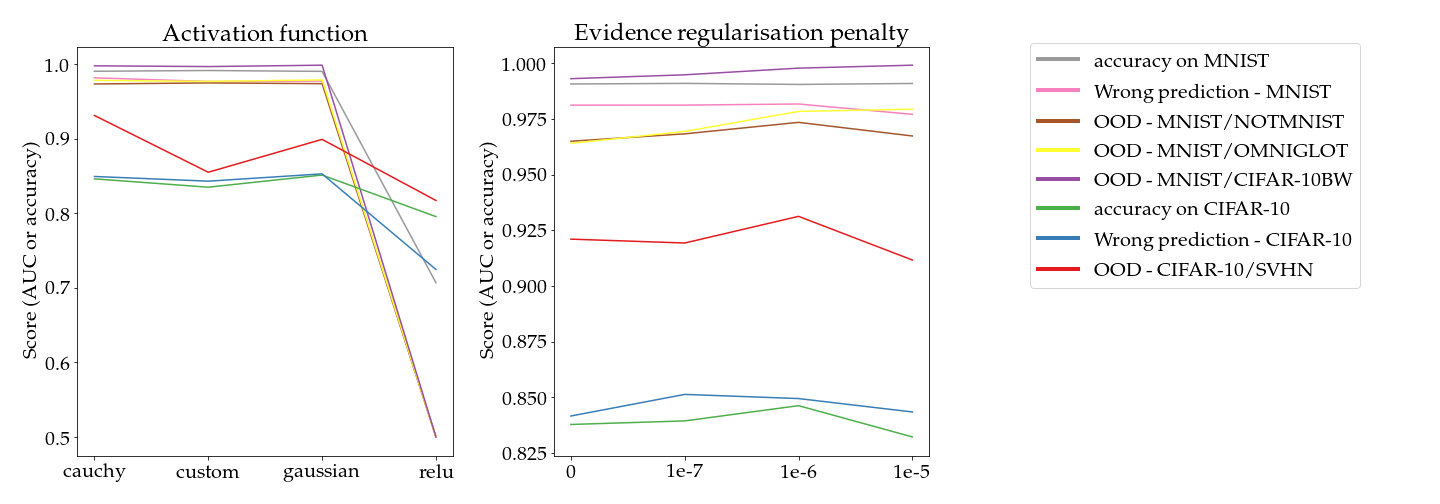}
\captionof{figure}{\label{fig:ablation}Results of the ablation experiments. The plots show the wrong prediction experiment's ROC AUC, out of-distribution experiment's ROC AUC and accuracies on \textit{MNIST} and \textit{CIFAR} datasets. We check the performance with changed \textit{l2} penalty (left), changed activation function (middle) and changed activity regularization penalty (right).}

\end{strip}

\section*{Appendix 1 - Ablation Study}
\label{appendix1}

We show the performance of our methods in the experiments on \textit{CIFAR-10} and \textit{MNIST} datasets if the hyperparameters are changed \hyperref[fig:ablation]{(Figure 4)}. The results are averages over three runs of experiments. The evidence regularisation penalty is important. The networks without it performed worse on all the checked tasks with an exception of wrong prediction detection on \textit{MNIST}. With too much of the regularization, the networks are unable to fit the data well. It results in a drop in results of all experiments on \textit{CIFAR-10}. We show also that it is possible to replace the rescaled Cauchy PDF function with another kernel function. Here, we show a comparison with rescaled Gaussian PDF ($\exp{\frac{-x^{2}}{2}}$) and a custom nonlinear function:

\begin{equation}
f(x) = 
\begin{cases}
    min(x+1,-x+1),& \text{if } |x| < 1\\
    0,              & \text{otherwise}
\end{cases}
\end{equation}

Still, non-kernel activation functions like ReLU are not able to correctly perform the out-of-distribution detection tasks.

\section*{Appendix 2 - Experiments details \& Preprocessing}

\textit{Omniglot} consists of black letters on a white background. We negated the images so that they resemble more the images from \textit{MNIST}. Without the negation, all the methods performed very well (between 0.999 and 1 in ROC AUC) on the out-of-distribution detection task.

In the sentiment analysis task, before feeding the data to the networks we preprocessed it by removing stopwords and words that did not occur in the pretrained embedding. We use a pretrained embedding in order to model vocabulary that exists in the out-of-distribution sets and was not present in the in-distribution dataset.

Regarding the baseline publication \citep{HendrycksGimpel16}: we were able to corroborate the results on \textit{IMDB/Movie Reviews} experiment when we split the observations from \textit{Movie Reviews} into single lines and use the same randomly initialized embeddings. The model was trained on full reviews from \textit{IMDB}. We argue that in such setting the use of average pooling after the embedding invalidates the experiment. The input is padded with zeros to 400 words. Now, if the sentence is very short, say 10 words, the true average of the embed words will be diminished by all the zeros after the sentence. Thus, the uncertainty estimation method needs only to correctly work in a very narrow region centred at zero in order to achieve high scores in the experiment.


For the state-of-the-art methods we compared with we made the following choices:

\begin{itemize}
\item Deep Ensembles - we skipped adversarial training, as adversarial training is a way to improve the performance of any of the methods used in the paper. We use an ensemble of 5 base networks.
\item Monte Carlo Dropout - for \textit{MNIST} we use dropout probability 0.25 on all but last layers, 0.5 on the only trainable layer in the sentiment experiment, and on \textit{CIFAR-10} network 0.25 only on the last but one layer. In larger networks setting dropout on many layers required a greater number of epochs to achieve top performance. We run 50 forward passes for variational prediction.
\item Bayes By Backprop - we observed that there is a trade-off between accuracy and OOD detection performance that depends on the initialisation of the variance. We chose initialisation that led to the best combination of accuracy and OOD detection performance in our view. We run 10 forward passes for variational prediction.
\end{itemize}

We followed the original publications when possible. For example, the number of networks in DE and number of inferences in BBP and MCD is taken from the original descriptions of the algorithms.

\section*{Appendix 3 - Visualization}

In the visualisation section of the paper, the uncertainties were normalised so that the predictive entropy and IS' probabilities could be visually compared. The normalisation for a method was performed by ranking the uncertainties and splitting them into 400 equal bins. Then, the bins are plotted. White colour represents the bin with the most uncertainty, the black - with the least.


For a better understanding of the latent space, we visualise the images decoded from the grid from the latent space \hyperref[fig:decoded]{(Figure 3)}.

\begin{strip}
\captionsetup{type=figure}
\centering
\begin{subfigure}{.5\textwidth}
\includegraphics[page=1,width=\textwidth]{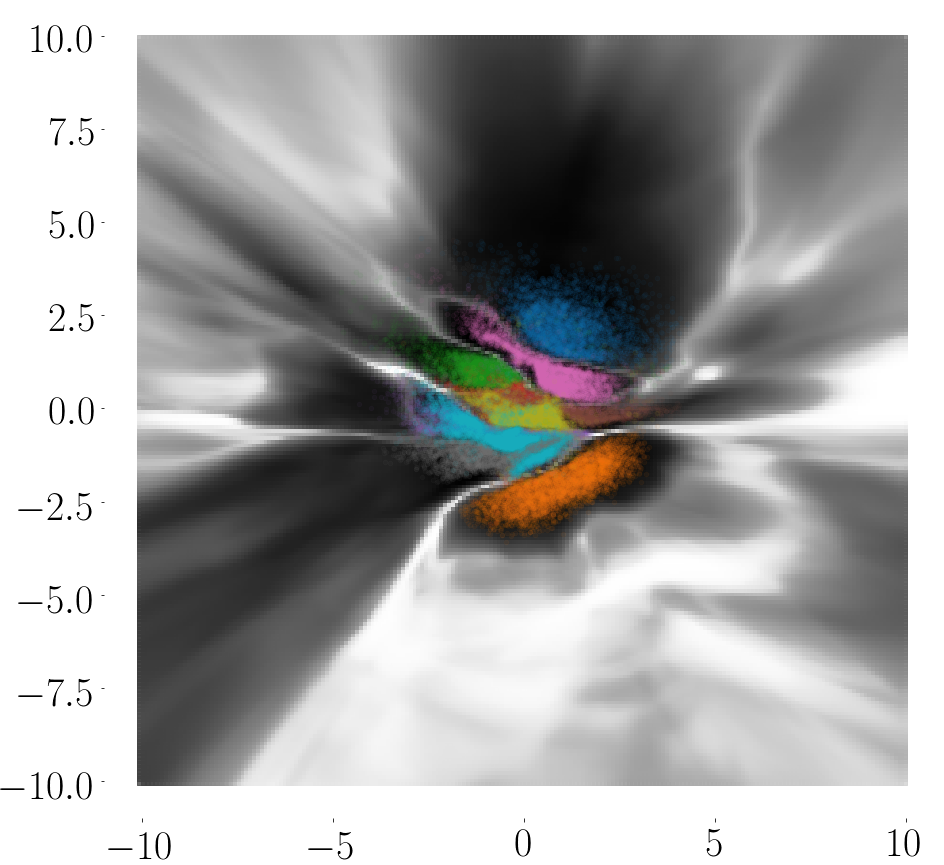}\\
\end{subfigure}%
\begin{subfigure}{.5\textwidth}
\includegraphics[page=2,width=\textwidth]{vae_decoded}\\
\end{subfigure}

\captionof{figure}{\label{fig:decoded}Visualisation of the images generated from the grid in the latent space (right) next to the uncertainty measure visualisation (left).}
\end{strip}

\section*{Appendix 4 - Mathematical derivations}
\label{appendix}

\begin{lemma}
\label{eq:1}
Let us assume that $x$ - the input to the $IS_a$ function is a result of a linear layer applied to the penultimate layer of a network (or input in case of networks without hidden layers) $x_p$, namely:

\begin{equation}
    x_i = \theta_i x_p + b_i.
\end{equation}

Then the following holds:

\begin{equation}
    \begin{split}
    \dfrac{\delta \log P_a^c(x)}{\delta b_i} = S(x)_i - IS_a(x)_i.   
    \end{split}
\end{equation}

\end{lemma}

\begin{proof}
Let us notice that:
\begin{equation}
\begin{split}
    \dfrac{\delta \log P_a^c(x)}{\delta x_i} =\frac{\delta \log \frac{\sum_{j=1}^{n}\exp x_j}{\sum_{j=1}^{n}\exp x_j + \exp a}}{\delta x_i} = \\
    \frac{\delta \log \sum_{j=1}^{n}\exp x_j}{\delta x_i}  - \\ \frac{\delta\log\left( \sum_{i=j}^{n}\exp x_j + \exp a\right)}{\delta x_i} =\\
    \frac{\delta \left(\sum_{j=1}^{n}\exp x_j\right)}{\delta x_i} \frac{1}{\sum_{j=1}^{n}\exp x_i} - \\ \frac{\delta \left(\sum_{j=1}^{n}\exp x_j + \exp a\right)}{\delta x_i} \frac{1}{\sum_{j=1}^{n}\exp x_j + \exp a} = \\
    \frac{\exp x_i}{\sum_{j=1}^{n}\exp x_j} - \frac{\exp x_i}{\sum_{j=1}^{n}\exp x_j + \exp a} = \\ S(x)_i - IS_a(x)_i.
\end{split}
\end{equation}
Now - the following equation holds:
\begin{equation}
    \frac{\delta x_i}{\delta b_i} = \frac{\delta (\theta_i x_p + b_i)}{\delta b_i} = 1.
\end{equation}
So:
\begin{equation}
\begin{split}
    \dfrac{\delta \log P_a^c(x)}{\delta b_i} = \dfrac{\delta \log P_a^c(x)}{\delta x_i}\frac{\delta x_i}{\delta b_i} = \\
    \\ S(x)_i - IS_a(x)_i.
\end{split}
\end{equation}
What completes the proof.
\end{proof}

\begin{lemma}
\label{eq:2}
Let us assume that $x$ - the input to the $IS_a$ function is a result of a linear layer applied to the penultimate layer of a network (or input in case of networks without hidden layers) $x_p$, namely:

\begin{equation}
    x_i = \theta_i x_p + b_i.
\end{equation}

Now, let $l_S(x, t)$ will be a \textit{cross-entropy} loss for softmax activation with $x$ - input to softmax and $t$ - a true class. Then the following equation holds:

\begin{equation}
    \dfrac{\delta l_{S}(x, t)}{\delta b_i} = S(x)_i - \mathbb{I}_{i = t},
\end{equation}
\end{lemma}

\begin{proof}
Let us notice that:

\begin{equation}
\begin{split}
    \dfrac{\delta l_{S}(x, t)}{\delta x_i} = \\
    \frac{-\delta\log\left(\frac{\exp x_i}{\sum_{j = 1}{n} \exp x_j}\right)}{\delta x_i} = \\
    -\frac{\delta x_t}{\delta x_i} + \frac{\delta \log \sum_{j=1}^{n}\exp x_j}{\delta x_i} = \\
    -\mathbb{I}_{i = t} + \frac{\delta \sum_{j=1}^{n}\exp x_j}{\delta x_i}\frac{1}{\sum_{j=1}^{n}\exp x_j} = \\
    -\mathbb{I}_{i = t} + \frac{\exp x_i}{\sum_{j=1}^{n}\exp x_j} = \\
    S(x)_i - \mathbb{I}_{i = t}
\end{split}
\end{equation}

Now - the following equation holds:
\begin{equation}
    \frac{\delta x_i}{\delta b_i} = \frac{\delta (\theta_i x_p + b_i)}{\delta b_i} = 1.
\end{equation}
So:
\begin{equation}
    \begin{split}
        \dfrac{\delta l_{S}(x, t)}{\delta b_i} &= \dfrac{\delta l_{S}(x, t)}{\delta x_i}\dfrac{\delta x_i}{\delta b_i} =\\
    &S(x)_i - \mathbb{I}_{i = t},
    \end{split}
\end{equation}
What completes the proof.

\end{proof}

\section*{Appendix 5 - No evidence regularisation scenario}

It is possible to use the Inhibited Softmax without the evidence regularisation. We present a table with comparison between IS, BASE and IS without the evidence regularisation \hyperref[tab:noer]{(Table 6)}. While the performance of the IS without the evidence regularisation does not match the other methods in the benchmark, it still outperforms the baseline in the out-of-distribution detection. If additional hyperparameter is a concern, one might drop it sacrificing the uncertainty estimation performance of IS.

\begin{table}
\begin{tabular}{ l | c | c | c }
\textbf{Task} & \textbf{IS} & \textbf{ISnoER} & \textbf{BASE}  \\\hline
MNIST/NOTMNIST & 0.973 & 0.966 & 0.958 \\\hline
MNIST/OMNIGLOT & 0.978 & 0.964 & 0.956 \\\hline
MNIST/CIFAR-10BW & 0.998 & 0.994 & 0.997 \\\hline
CIFAR-10/SVNH & 0.931 & 0.921 & 0.866 \\\hline
CIFAR-10/LFW-A & 0.74 & 0.715 & 0.593 \\\hline
acc MNIST & 0.991 & 0.991 & 0.992 \\\hline
acc CIFAR & 0.846 & 0.838 & 0.851
\end{tabular}%

\caption{\label{tab:noer} Comparison of the Inhibited Softmax with and without the evidence regularisation (resp. IS, ISnoER) and the baseline on the out-of-distribution detection and predictive performance tasks. The metrics are ROC AUC for the out-of-distribution detection and accuracy for the predictive performance.}
\end{table}
\end{document}